\documentclass[11pt,letterpaper]{article}
\usepackage[T1]{fontenc}
\usepackage[utf8]{inputenc}
\usepackage[margin=1in]{geometry}
\usepackage{mathtools}
\usepackage{amsthm}
\usepackage{libertine}
\usepackage[libertine,vvarbb]{newtxmath}
\usepackage[scaled=0.92]{inconsolata}
\usepackage[final]{microtype}
\usepackage{booktabs}
\usepackage{placeins}
\usepackage[round,authoryear]{natbib}
\usepackage{xcolor}
\usepackage{xurl}
\usepackage{hyperref}
\usepackage{bookmark}

\newcommand{\emailaddress}[1]{{\fontfamily{zi4}\fontseries{m}\fontshape{n}\selectfont #1}}

\definecolor{LinkColor}{RGB}{35,87,137}
\definecolor{CiteColor}{RGB}{156,66,42}
\hypersetup{
  colorlinks=true,
  linkcolor=LinkColor,
  citecolor=CiteColor,
  urlcolor=LinkColor,
  bookmarksnumbered=true,
  breaklinks=true
}

\allowdisplaybreaks

\theoremstyle{plain}
\newtheorem{theorem}{Theorem}
\newtheorem{lemma}[theorem]{Lemma}
\newtheorem{proposition}[theorem]{Proposition}
\newtheorem{corollary}[theorem]{Corollary}
\theoremstyle{definition}
\newtheorem{definition}[theorem]{Definition}

\newenvironment{keywords}
  {\par\smallskip\noindent\textbf{Keywords: }\ignorespaces}
  {\par\medskip}

\newcommand{\R}{\mathbb R}
\newcommand{\E}{\mathbb E}
\renewcommand{\P}{\mathbb P}
\newcommand{\Hc}{\mathcal H_{n,s,k}^{W,B}(\mathcal X_0)}
\newcommand{\Rad}{\mathcal R}
\newcommand{\rhoM}{\rho_m}
\DeclarePairedDelimiter{\norm}{\lVert}{\rVert}

\DeclarePairedDelimiterX{\ip}[2]{\langle}{\rangle}{#1,#2}
\newcommand{\ind}{\mathbf 1}

\newcommand{\ellm}{\ell_m}

\newcommand{\Ball}{\mathbb B_R^n}
\newcommand{\Hball}{\mathcal H_{n,s,k}^{W,B}(\Ball)}
\DeclareMathOperator{\clip}{clip}
\DeclareMathOperator{\KL}{KL}
\DeclareMathOperator{\TV}{TV}

\DeclareMathOperator{\spn}{span}

\title{\Large\bfseries Nearly Tight Rademacher Bounds for Sparsely Activated\\Neural Networks}
\author{%
  \normalsize
  Xiaoyu Li\textsuperscript{1} \quad
  Zhizhou Sha\textsuperscript{2} \quad
  Jiaojiao Jiang\textsuperscript{1} \quad
  Junbin Gao\textsuperscript{3} \quad
  Andi Han\textsuperscript{3}\\[0.8em]
  \small
  \begin{tabular}{@{}l@{\hspace{1.2em}}l@{}}
    \textsuperscript{1}University of New South Wales &
      \emailaddress{\{xiaoyu.li2,jiaojiao.jiang\}@unsw.edu.au}\\
    \textsuperscript{2}University of Texas at Austin &
      \emailaddress{zhizhousha@utexas.edu}\\
    \textsuperscript{3}University of Sydney &
      \emailaddress{\{junbin.gao,andi.han\}@sydney.edu.au}
  \end{tabular}%
}
\date{}
\hypersetup{pdftitle={Nearly Tight Rademacher Bounds for Sparsely Activated Neural Networks},pdfauthor={Xiaoyu Li, Zhizhou Sha, Jiaojiao Jiang, Junbin Gao, Andi Han}}
\begin{document}
\maketitle
\begin{abstract}
An input may activate few hidden units even when different inputs collectively use an entire
network. We study the statistical complexity of this input-dependent sparsity in the one-hidden-layer
ReLU model of \citeauthor{awasthi2024sparse} (COLT~2024). For width $s$, at most $k$ active units per input,
and effective weight and bias bounds $W,B$, every size-$m$ sample in the class's fixed
radius-$R$ input domain satisfies
$\mathcal R(S)\le CWR\min\{k,\sqrt{sk/m}\log^{3/2}(2m)\}+kB/\sqrt m$.
A support-preserving cover and a single normalized chaining argument remove the previous explicit
dimension factor, up to logarithms. Lower bounds on appropriate i.i.d. marginals match up to
those logarithms, showing how changing active units across inputs retains a width dependence.
The input domain matters: zero-bias networks sparse on the entire ball have at most $2k$ nonzero
units and complexity $O(kWR/\sqrt m)$, whereas bias bounds comparable to $WR$ restore the worst-case rate
on that same domain in only logarithmic dimension. A spherical-cap construction proves the latter
claim without assuming sparsity merely on the sampling support.
For a specified normalized bounded loss and biases comparable to $WR$, we also obtain agnostic
minimax excess-risk bounds of order $\min\{1,\sqrt{s/(km)}\}$ up to logarithms.
\end{abstract}
\begin{keywords}
Activation sparsity, Rademacher complexity, metric entropy, neural networks, generalization
\end{keywords}
\section{Introduction}
\label{sec:intro}

Activation sparsity limits how many hidden units a network uses on an input, while allowing the
identities of those units to change across inputs. It limits the number of nonzero contributions
without simply reducing the number of available parameters. The statistical question is whether this
input-dependent constraint controls a network's capacity uniformly over all admissible activation
patterns.

\citet{awasthi2024sparse}, henceforth ADKM, formulated this question for one-hidden-layer ReLU
networks. Their model is motivated by sparse hidden activations, but imposes an exact mathematical
promise: at most $k$ of the $s$ hidden units are active on every point in a prescribed input set.
For bounded weights and inputs, they obtained a Rademacher bound with dependence
$\widetilde O(\sqrt{snk/m})$, where $n$ is the input dimension, and conjectured that the factor
$\sqrt n$ could be removed (Conjecture~20). The challenge is to retain the sparsity improvement
in width while avoiding a count of all halfspace activation patterns.

We establish the dimension-free rate up to $\log^{3/2}(2m)$ and ask when its remaining width
dependence is necessary. Our analysis separates three questions: the worst-case capacity of the
fixed-domain class, the effect of the domain on admissible activation regions, and the agnostic
learning difficulty under a specified loss. The answers distinguish sparse computation on each
input from a globally small collection of useful neurons.

\subsection{Model and results}

Let $n,s,m\ge1$, $1\le k\le s$, and $W,B,R\ge0$. Fix a nonempty set
$\mathcal X_0\subseteq\{x\in\R^n:\norm{x}_2\le R\}$ before sampling.
Write $\sigma(z)=\max\{z,0\}$.

\begin{definition}[ADKM's class on a fixed input set]
\label{def:class}
The class $\Hc$ consists of restrictions to $\mathcal X_0$ of functions
\[
 h(x)=\sum_{j=1}^s u_j\sigma(\ip{w_j}{x}-b_j)
\]
admitting a representation with
\[
 \norm{u}_\infty\max_j\norm{w_j}_2\le W,\qquad
 \norm{u}_\infty\max_j|b_j|\le B,\qquad
 \#\{j:\ip{w_j}{x}>b_j\}\le k\quad(x\in\mathcal X_0).
\]
\end{definition}

The input set is part of the class definition. Upper bounds hold on every
$S=(x_1,\ldots,x_m)\in\mathcal X_0^m$; a lower-bound construction is allowed to choose
$\mathcal X_0$ and a distribution on it. Requiring sparsity on a larger set can give a smaller class.
No assertion below identifies the support-wide class with a class chosen after seeing a sample.

For independent uniform signs $\zeta_i\in\{-1,1\}$, define
\[
 \Rad_{\mathcal H}(S)=\frac1m\E_\zeta\sup_{h\in\mathcal H}\sum_{i=1}^m\zeta_i h(x_i),
 \qquad A=WR,\quad \ellm=\log(2m),\quad
 \rhoM=\frac1m\E\left|\sum_{i=1}^m\zeta_i\right|.
\]
All logarithms are natural unless indicated otherwise. The elementary moment bound in
Lemma~\ref{lem:moment} gives $1/\sqrt{3m}\le\rhoM\le1/\sqrt m$.

\begin{theorem}[Dimension-free upper bound with separated bias]
\label{thm:upper}
There is a universal constant $C\ge1$ such that, for all parameters and input sets above and
every $S\in\mathcal X_0^m$,
\begin{equation}
\label{eq:headline}
 \Rad_{\Hc}(S)
 \le C A\min\left\{k,\sqrt{\frac{sk}{m}}\ellm^{3/2}\right\}+kB\rhoM.
\end{equation}
In addition, $\Rad_{\Hc}(S)\le k(A+B)$ and
$\Rad_{\Hc}(S)\le2s(A+B)/\sqrt m$.
\end{theorem}

Thus one may always take the smallest of the three upper bounds. The last one is useful in the
dense regime $k=s$, where it avoids a logarithmic loss. Dimension-free means that $n$ does not
appear separately from $R$; on the unnormalized Boolean cube, $R=\sqrt n$ still carries a
dimension dependence.

\begin{theorem}[Lower bound on i.i.d. samples]
\label{thm:lower}
Let $q=\min\{\lfloor s/k\rfloor,m\}$ and suppose $n\ge q$. There are a finite input set
$\mathcal X_0$ in the radius-$R$ ball and a marginal $D_x$ on it such that
\begin{equation}
\label{eq:lower}
 \E_{S\sim D_x^m}\Rad_{\Hc}(S)
 \ge\frac1{4\sqrt2}\left(
 A\min\left\{k,\sqrt{\frac{sk}{m}}\right\}
 +\frac{kB}{\sqrt m}\right).
\end{equation}
The same expression lower-bounds the supremum over all admissible input sets and samples.
\end{theorem}

For $n\ge q$ and $m\ge s/k$, these results determine the worst-case scaling
$\widetilde\Theta(A\sqrt{sk/m}+kB/\sqrt m)$, with logarithmic slack only in the weight
term. For $m\le s/k$ and sufficiently large $n$, the weight term saturates at $kA$.
The bounds cover arbitrary $s,k,m$ without divisibility assumptions. They concern class complexity;
Theorem~\ref{thm:lower} alone is not a lower bound on every learning algorithm.

\paragraph{The domain can change the answer.}
On the entire Euclidean ball, zero-bias $k$-sparsity forces at most $2k$ nonzero units, giving
the smaller rate $\Theta(kA/\sqrt m)$ in the worst case
(Proposition~\ref{prop:zero-ball}). Positive thresholds restore the width dependence:
Theorem~\ref{thm:ball-lower} realizes the lower rate on that same full-ball class, using
separated spherical caps in only logarithmic dimension when $B\ge(\sqrt3/2)A$.
Thus the lower bound need not rely on imposing sparsity only at isolated sampling locations.

\paragraph{A same-loss learning characterization.}
We derive general bounded-loss guarantees and, separately, prove an agnostic minimax result for
the normalized class and a fixed bounded linear loss. Under the explicit dimension and
comparable-bias conditions of Theorem~\ref{thm:agnostic-minimax}, its expected excess-risk rate is
$\widetilde\Theta(\min\{1,\sqrt{s/(km)}\})$.
This lower bound is established directly in the learning experiment, not inferred from
Rademacher complexity. It is not a claim of optimality for every loss or for realizable learning.

\paragraph{Comparison with the original bound.}
ADKM's equation~(15) has the displayed dependence
$(WR+B)\sqrt{snk\log(km(R+B))}/\sqrt m$ in its parameter regime. Their conjectured
dependence is $(WR+B)\sqrt{sk/m}$. Equation~\eqref{eq:headline} removes the explicit
$\sqrt n$ at a logarithmic cost and improves the bias coefficient from $\sqrt{sk}$ to $k$.
The logarithm-free form of their conjecture remains unresolved here. Bounds based only on the
sum of the individual neuron complexities give $O(s(A+B)/\sqrt m)$; the point is to combine
norm control with the shared activation budget.

\subsection{Technique overview}
\label{sec:techniques}

The upper bound uses the total activation budget without fixing an activation pattern.
The lower bounds then ask how many independently signed neuron clusters that budget permits,
first on a discrete domain and then on the entire ball.

\paragraph{From a network to one normalized neuron process.}
By positive homogeneity, absorb each outer weight's magnitude into its neuron, leaving an outer
sign and parameters $\|v_j\|_2\le W$, $|\beta_j|\le B$. If $I_j$ is its active sample set,
double-counting gives $\sum_j|I_j|\le km$. Let $\mathcal F_t$ be the sample-output vectors of
one neuron active on at most $t$ coordinates, and put
$Q=\max_{1\le t\le m}t^{-1/2}\sup_{f\in\mathcal F_t}|\ip\zeta f|$. Then
\[
 \sum_j\left|\sum_i\zeta_i\sigma(\ip{v_j}{x_i}-\beta_j)\right|
 \le Q\sum_j\sqrt{|I_j|}\le Q\sqrt{skm}.
\]
Inactive units contribute zero. After division by $m$ and averaging, Lemma~\ref{lem:budget}
bounds the network complexity by $\sqrt{sk/m}\,\E Q$. The displayed inequality holds for each sign vector:
it does not require choosing the active sets before observing the signs. The remaining task is
to control all activation counts simultaneously.

\paragraph{A cover that preserves inactive coordinates.}
A direct affine approximation followed by ReLU can turn zero coordinates into small positive
ones, losing the sparse support in its Euclidean error. We instead shift every approximating
pre-activation downward. If $\|g-\widehat g\|_\infty\le\delta$, then
$\widehat f=\sigma(\widehat g-\delta)$ vanishes wherever $f=\sigma(g)$ vanishes and differs
from $f$ by at most $2\delta$ elsewhere. Hence $\|f-\widehat f\|_2\le2\delta\sqrt t$
for $f\in\mathcal F_t$. Applying this transformation to a dimension-free affine cover gives
Lemma~\ref{lem:cover}, without counting halfspace activation patterns. The centers are allowed
outside the parameter-constrained class; only their approximation error matters.

\paragraph{One chain, rather than separate bounds at each count.}
Normalize $\mathcal F_t$ by $\sqrt t$ and take the symmetric union
\[
 T=\bigcup_{t=1}^m\left(t^{-1/2}\mathcal F_t\cup-t^{-1/2}\mathcal F_t\right),
 \qquad Q=\sup_{z\in T}\ip\zeta z.
\]
The normalization removes $t$ from the cover's scale-dependent entropy; selecting a member
of the union costs only an additive $\log(2m)$. A single finite chaining argument therefore
bounds the expectation of the maximum itself, without interchanging maximum and expectation.
Truncating the chain at radius at most $(A+B)/\sqrt m$ controls its residual by $A+B$.
The affine entropy contributes $\sqrt{\log(2m)}$, and summing over scales contributes another
logarithm, giving $\E Q=O((A+B)\log^{3/2}(2m))$. This identifies the remaining logarithmic
loss: removing the lower-order union cost alone would not remove it.

\paragraph{Clipping thresholds separates the constant part.}
The preceding argument initially charges both $A$ and $B$ at the network rate. To refine it,
clip $\beta$ to $\beta^\circ\in[-A,A]$ and use the identity
\[
 \sigma(\ip v x-\beta)=\sigma(\ip v x-\beta^\circ)+(-\beta-A)_+
 \qquad(\|x\|_2\le R).
\]
Clipping creates no new activation. Every nonzero constant correction comes from a unit
previously active everywhere, so at most $k$ corrections occur. Thus $h=h_0+c$, with clipped
bias bound $\min\{A,B\}$ and $|c|\le k(B-A)_+$ (Lemma~\ref{lem:bias}). Applying the chain
and the envelope bound to $h_0$, and the exact constant-class complexity to $c$, yields
Theorem~\ref{thm:upper}.

\paragraph{Independent clusters and the geometry of the domain.}
For the discrete lower bound, use $q=\min\{\lfloor s/k\rfloor,m\}$ orthogonal input directions
and $k$ parallel neurons per direction. Each cluster carries its own sign. Under uniform
i.i.d. sampling, its signed occupancy has absolute expectation of order $\sqrt{m/q}$, giving
the weight lower bound without balanced-count or divisibility assumptions. The separate
constant subclass supplies the bias term. On the full ball, however, a generic antipodal pair
counts every nonzero zero-bias neuron, forcing at most $2k$ in total.
To recover independent clusters there, we use disjoint spherical caps. Directions with pairwise
inner product at most $1/2$, together with threshold $\tau A$ for $\tau=\sqrt3/2$, prevent
two clusters from activating at any point of the ball. A random-sign packing supplies the
directions in $O(\log(2q))$ dimensions. When $B\ge\tau A$, each cap center still has output
amplitude $(1-\tau)kA$, so the same occupancy argument applies
(Theorem~\ref{thm:ball-lower}).

\paragraph{Turning sign encoding into a learning lower bound.}
A complexity lower bound alone does not establish learning hardness. Under the conditions of
Theorem~\ref{thm:agnostic-minimax}, normalization by $k(A+B)$ makes the cap construction's
sign amplitude a positive universal constant. Assign binary labels small unknown means
$\delta\theta_b$ at the $q$ cap centers. Adjacent sign choices have joint-sample KL
divergence at most $O(m\delta^2/q)$. Choosing $\delta$ as a sufficiently small constant
times $\sqrt{q/m}$ keeps adjacent distributions hard to distinguish. Total variation then
limits any learner's correlation with the signs, including randomized and improper learners.
For the specified bounded linear loss, this gives an expected excess-risk lower bound of
order $\sqrt{q/m}$. Approximate ERM and the complexity upper bound give the corresponding
upper rate for that same loss and normalized class; Appendix~\ref{app:agnostic} supplies the
testing calculation.

\subsection{Related work}

Our comparison uses the support-wide class and product norms of \citet{awasthi2024sparse}.
Their expressivity and computational results are distinct from the capacity question studied
here. Norm-based analyses such as \citet{neyshabur2015norm,golowich2018size} control complexity
without imposing this activation promise. We use an elementary neuronwise bound as the
dimension-free baseline rather than asserting that all such norm-based results have identical
specializations.

Other sparsity notions lead to different guarantees. \citet{galanti2023norm} study
compositionally sparse networks, in which each neuron has a limited number of inputs.
\citet{muthukumar2023sparsity} use activation stability and sensitivity analysis to obtain
predictor-dependent, derandomized PAC-Bayes guarantees across multiple layers. Our result is a
uniform complexity bound for an entire one-layer class with a fixed activation promise; it does
not require stability under parameter perturbations.

Scale-sensitive covers of norm-constrained linear classes have a long history
\citep{zhang2002covering}. The affine cover used here is an application of dual Sudakov
\citep{pajor1986subspaces,ledoux1991probability}; the chain uses the familiar multiscale
principle of \citet{dudley1967sizes}. The contribution of the covering step is its compatibility
with sparse supports. We include the finite chaining argument and a proof of the required
geometric covering estimate so that the proof interfaces can be checked directly.

\section{A dimension-free upper bound}
\label{sec:upper}

We first prove the bound with scale $M=A+B$. We then separate the contribution of large
biases. Vectors of sample evaluations use the unnormalized Euclidean norm in $\R^m$.
The covering number $N(T,\epsilon,\|\cdot\|)$ permits centers in the ambient vector space;
all covers used below are finite and deterministic once the sample is fixed.

\subsection{Reduction and the activation budget}

Positive homogeneity gives the exact reduced representation
\begin{equation}
\label{eq:reduced}
 h(x)=\sum_{j=1}^s\varepsilon_j\sigma(\ip{v_j}{x}-\beta_j),\qquad
 \varepsilon_j\in\{-1,1\},\quad \|v_j\|_2\le W,\quad |\beta_j|\le B.
\end{equation}
For $u_j\ne0$, take $v_j=|u_j|w_j$, $\beta_j=|u_j|b_j$, and
$\varepsilon_j=\operatorname{sign}(u_j)$. For $u_j=0$, take $v_j=0$, $\beta_j=0$ and
$\varepsilon_j=1$; this only decreases the activation count. Conversely,
\eqref{eq:reduced} is realized in Definition~\ref{def:class} by $u_j=\varepsilon_j$,
$w_j=v_j$, and $b_j=\beta_j$. Sparsity does not depend on the outer signs.

For $t\in[m]$, let
\[
 \mathcal F_t=\left\{(\sigma(\ip v{x_i}-\beta))_{i=1}^m:
 \|v\|_2\le W,\ |\beta|\le B,\ \#\{i:\ip v{x_i}>\beta\}\le t\right\},
\]
and set $P(t)=\sup_{f\in\mathcal F_t}|\ip\zeta f|$ and
$Q=\max_{t\in[m]}P(t)/\sqrt t$. Each $f\in\mathcal F_t$ has entries in $[0,M]$
and satisfies $\|f\|_2\le M\sqrt t$.

\begin{lemma}[Sample relaxation and budget]
\label{lem:budget}
For every $S\in\mathcal X_0^m$,
\begin{equation}
\label{eq:budget}
 \Rad_{\Hc}(S)\le\sqrt{\frac{sk}{m}}\,\E_\zeta Q.
\end{equation}
\end{lemma}
\begin{proof}
Enlarge the reduced parameter set by requiring sparsity only on $S$. For fixed parameters,
maximizing over the outer signs yields
$\sum_j|N_j|$, where $N_j=\sum_i\zeta_i\sigma(\ip{v_j}{x_i}-\beta_j)$.
Let $t_j=\#\{i:\ip{v_j}{x_i}>\beta_j\}$. A unit with $t_j=0$ contributes zero;
otherwise $|N_j|\le Q\sqrt{t_j}$. Double-counting active sample-unit pairs gives
$\sum_jt_j\le km$. Therefore, for every sign vector and admissible configuration,
\[
 \sum_j|N_j|\le Q\sum_j\sqrt{t_j}\le Q\sqrt{s\sum_jt_j}\le Q\sqrt{skm}.
\]
Take the supremum, divide by $m$, and average over the signs.
\end{proof}

\subsection{Covering sparse outputs}

The affine evaluation class
$\mathcal A=\{(\ip v{x_i}-\beta)_i:\|v\|_2\le W,\ |\beta|\le B\}$ obeys
\begin{equation}
\label{eq:affine}
 \log N(\mathcal A,\delta,\|\cdot\|_\infty)
 \le80\frac{A^2\log(2m)}{\delta^2}+\log(1+4B/\delta),\qquad\delta>0.
\end{equation}
For completeness, Lemma~\ref{lem:affine} proves this estimate using a Gaussian packing
argument. Restricting to $\spn\{x_i\}$ deals with the fact that
$\max_i|\ip v{x_i}|$ need only be a seminorm on the original space.

\begin{lemma}[Shifted cover]
\label{lem:cover}
For $t\in[m]$ and $\epsilon>0$,
\begin{equation}
\label{eq:sparse-cover}
 \log N(\mathcal F_t,\epsilon,\|\cdot\|_2)
 \le320\frac{A^2t\log(2m)}{\epsilon^2}
      +\log(1+8B\sqrt t/\epsilon).
\end{equation}
\end{lemma}
\begin{proof}
Put $\delta=\epsilon/(2\sqrt t)$ and cover $\mathcal A$ in sup norm at radius $\delta$.
For $f=\sigma(g)\in\mathcal F_t$, choose a center $\widehat g$ with
$\|g-\widehat g\|_\infty\le\delta$ and use $\widehat f=\sigma(\widehat g-\delta)$.
If $g_i\le0$, then $\widehat g_i-\delta\le0$, so $f_i=\widehat f_i=0$.
On the other coordinates, the $1$-Lipschitz property of ReLU gives
$|f_i-\widehat f_i|\le2\delta$. Thus
$\|f-\widehat f\|_2\le2\delta\sqrt t=\epsilon$. The transformed cover has no more
centers than the affine cover; substituting $\delta$ in \eqref{eq:affine} proves the claim.
The centers need not satisfy the bias constraint, which is why we use external covers.
\end{proof}

\subsection{One chain over all normalized scales}

Define the symmetric set
\[
 T=\bigcup_{t=1}^m\left(t^{-1/2}\mathcal F_t\ \cup\ -t^{-1/2}\mathcal F_t\right).
\]
It contains $0$, has radius at most $M$, and satisfies
$Q=\sup_{z\in T}\ip\zeta z$. By scaling Lemma~\ref{lem:cover} and taking the union of
$2m$ covers, we get the uniform estimate
\begin{equation}
\label{eq:normalized-cover}
 \log N(T,\epsilon,\|\cdot\|_2)
 \le\ellm+320\frac{A^2\ellm}{\epsilon^2}+\log(1+8B/\epsilon).
\end{equation}
The entropy cost of selecting an activation count is only $\log(2m)$, separate from the
$\epsilon^{-2}$ term. No comparison of the expectations of different $P(t)$ is needed.

\begin{lemma}[Finite Rademacher chaining]
\label{lem:chain}
Let $T\subseteq\R^m$ have radius at most $D>0$ and put $r_j=D2^{-j}$.
For an integer $J\ge1$, suppose external $r_j$-nets exist with cardinalities $N_j$ and
$\log N_j\le H_j$, where $0\le H_1\le\cdots\le H_J$. Then
\begin{equation}
\label{eq:finite-chain}
 \E\sup_{z\in T}\ip\zeta z
 \le\sqrt m\,r_J+6\sum_{j=1}^Jr_j\sqrt{H_j}.
\end{equation}
\end{lemma}
\begin{proof}
Choose deterministic net maps $\pi_j$ with $\|z-\pi_jz\|_2\le r_j$ and set
$\pi_0z=0$, $N_0=1$. The edge set
$E_j=\{\pi_jz-\pi_{j-1}z:z\in T\}$ has at most $N_jN_{j-1}$ elements, each with
norm at most $r_j+r_{j-1}=3r_j$. For every deterministic vector $a$,
\[
 \E e^{\lambda\ip\zeta a}=\prod_i\cosh(\lambda a_i)
 \le e^{\lambda^2\|a\|_2^2/2}.
\]
The exponential-moment maximal inequality therefore gives
\[
 \E\max_{a\in E_j}\ip\zeta a\le3r_j\sqrt{2\log(N_jN_{j-1})}
 \le6r_j\sqrt{H_j}.
\]
If the edge set is a singleton, its expected maximum is zero,
so the same inequality applies. Telescope from $0$ to $\pi_Jz$ and use
$\ip\zeta{z-\pi_Jz}\le\sqrt m\,r_J$ for the residual. Taking the supremum and
expectation proves \eqref{eq:finite-chain}. The process is defined at every point of
$\R^m$, so external centers cause no change to the argument.
\end{proof}

\begin{proposition}[Unseparated bound]
\label{prop:base}
There is a universal constant $C_0\ge1$ such that
\[
 \E Q\le C_0M\ellm^{3/2},\qquad
 \Rad_{\Hc}(S)\le C_0M\sqrt{sk/m}\ellm^{3/2}.
\]
\end{proposition}
\begin{proof}
If $M=0$, both quantities vanish. Otherwise use Lemma~\ref{lem:chain} with $D=M$ and
$J=\max\{1,\lceil\log_2\sqrt m\rceil\}$. Then $\sqrt m\,r_J\le M$ and
$J\le\ellm/\log2$. Since $B\le M$, \eqref{eq:normalized-cover} permits
\[
 H_j=\ellm+320A^2\ellm/r_j^2+(j+4)\log2.
\]
Subadditivity of the square root, $\sum_{j\ge1}2^{-j}=1$, and
$\sum_{j\ge1}2^{-j}\sqrt{j+4}\le\sum_{j\ge1}2^{-j}(j+4)=6$ give
\begin{align*}
 \E Q
 &\le M+6M\sqrt{\ellm}+6\sqrt{320}\,A\sqrt{\ellm}\,J
                 +36M\sqrt{\log2}\\
 &\le C_0M\ellm^{3/2}.
\end{align*}
For example, one may take
\[
 C_0=\frac{6\sqrt{320}+6}{\log2}
       +\frac{1+36\sqrt{\log2}}{(\log2)^{3/2}}.
\]
The comparison uses only $A\le M$ and $\ellm\ge\log2$, including $m=1$.
Lemma~\ref{lem:budget} yields the network bound.
\end{proof}

\subsection{Large biases contribute constants}

\begin{lemma}[Bias clipping]
\label{lem:bias}
Put $b_0=\min\{B,A\}$ and $d=(B-A)_+$. Every $h\in\Hc$ can be written on
$\mathcal X_0$ as $h=h_0+c$, where
\[
 h_0\in\mathcal H_{n,s,k}^{W,b_0}(\mathcal X_0),\qquad |c|\le kd.
\]
\end{lemma}
\begin{proof}
Use \eqref{eq:reduced} and clip each $\beta_j$ to $[-A,A]$.
For $a\in[-A,A]$ and every $\beta\in\R$, the three cases
$\beta<-A$, $-A\le\beta\le A$, and $\beta>A$ give
\begin{equation}
\label{eq:bias-identity}
 \sigma(a-\beta)=\sigma(a-\beta^\circ)+(-\beta-A)_+,
 \qquad \beta^\circ=\max\{-A,\min\{\beta,A\}\}.
\end{equation}
Clipping creates no active unit at any input in the ball. A unit with $\beta<-A$ was
strictly active everywhere before clipping, and a unit with $\beta>A$ stays inactive.
Hence $h_0=\sum_j\varepsilon_j\sigma(\ip{v_j}x-\beta_j^\circ)$ belongs to the stated
class. Every nonzero summand of
$c=\sum_j\varepsilon_j(-\beta_j-A)_+$ comes from a unit active everywhere.
Because $\mathcal X_0$ is nonempty, there can be at most $k$ such units, each contributing
at most $d$ in absolute value.
\end{proof}

\subsection{Proof of Theorem~\ref{thm:upper}}
\label{sec:proof-upper}

\begin{proof}
The complexity of constants in $[-kd,kd]$ is $kd\rhoM$. Subadditivity and
Lemma~\ref{lem:bias} give the slightly sharper intermediate bound
\begin{equation}
\label{eq:refined-upper}
 \Rad_{\Hc}(S)
 \le (A+b_0)\min\{k,C_0\sqrt{sk/m}\ellm^{3/2}\}+kd\rhoM.
\end{equation}
Here the two estimates for the clipped class are its pointwise envelope and
Proposition~\ref{prop:base}. Using $A+b_0\le2A$, $d\le B$, and $C_0\ge1$ proves
\eqref{eq:headline} with $C=2C_0$. This also covers $A=0$: the clipped class is then zero
and the original class consists exactly of constants in $[-kB,kB]$.

The original envelope is $|h(x)|\le kM$, proving the first additional bound.
For the second, ignore sparsity and bound each signed neuron separately. Since the
unsigned neuron class contains zero, sign symmetry and scalar contraction imply
\[
 \E\sup_{\|v\|\le W,|\beta|\le B}
 \left|\sum_i\zeta_i\sigma(\ip v{x_i}-\beta)\right|
 \le2\left(W\E\left\|\sum_i\zeta_ix_i\right\|_2+B\E\left|\sum_i\zeta_i\right|\right)
 \le2M\sqrt m.
\]
Summing over $s$ neurons proves $2sM/\sqrt m$.
\end{proof}

\section{Lower bounds on the same class}
\label{sec:lower}

The lower bound uses parallel neurons within each cluster and orthogonal directions between
clusters. Parallelism allows $k$ units to each attain value $A$ on a single input while respecting
the radius constraint. Orthogonality makes the output values of different clusters independent
parameters on the constructed support.

\begin{lemma}[A moment estimate]
\label{lem:moment}
If $Z$ has finite fourth moment and $\E Z^2>0$, then
\[
 \E|Z|\ge\frac{(\E Z^2)^{3/2}}{(\E Z^4)^{1/2}}.
\]
In particular, $\E|\sum_{i=1}^m\zeta_i|\ge\sqrt{m/3}$.
\end{lemma}
\begin{proof}
H\"older applied to $|Z|^2=|Z|^{2/3}|Z|^{4/3}$ gives
$\E Z^2\le(\E|Z|)^{2/3}(\E|Z|^4)^{1/3}$. For the sign sum,
$\E Z^2=m$ and $\E Z^4=3m^2-2m\le3m^2$.
\end{proof}

\subsection{Proof of Theorem~\ref{thm:lower}}
\label{sec:proof-lower}

\begin{proof}
First suppose $A>0$. Choose orthonormal vectors $e_1,\ldots,e_q\in\R^n$ and let
$D_x$ be uniform on $\mathcal X_0=\{Re_1,\ldots,Re_q\}$. Use $q$ clusters of $k$
neurons each, with weight $We_b$ and zero bias in cluster $b$. There are at most $s$
neurons; pad with inactive zero units when necessary. On $Re_b$, only cluster $b$ is
active. Choosing a common outer sign $\theta_b\in\{-1,1\}$ in that cluster realizes
$h_\theta(Re_b)=kA\theta_b$.

Write $X_i=Re_{J_i}$, where the $J_i$ are i.i.d. uniform on $[q]$ and independent of
the Rademacher signs. For each realized sample, optimizing over the cluster signs gives
\begin{equation}
\label{eq:cluster-rad}
 \Rad_{\Hc}(S)\ge\frac{kA}{m}\sum_{b=1}^q
 \E_\zeta\left|\sum_{i=1}^m\zeta_i\ind\{J_i=b\}\right|.
\end{equation}
For a fixed $b$, set $Z_b=\sum_i\zeta_i\ind\{J_i=b\}$ and $\lambda=m/q\ge1$.
Independence and centering give
\[
 \E Z_b^2=\lambda,\qquad
 \E Z_b^4=\frac mq+\frac{3m(m-1)}{q^2}\le\lambda+3\lambda^2\le4\lambda^2.
\]
Lemma~\ref{lem:moment}, followed by \eqref{eq:cluster-rad}, yields
\begin{equation}
\label{eq:weight-lower}
 \E_S\Rad_{\Hc}(S)\ge\frac{kA}{2}\sqrt{q/m}
 \ge\frac{A}{2\sqrt2}\min\{k,\sqrt{sk/m}\}.
\end{equation}
For the last step, $\lfloor s/k\rfloor\ge s/(2k)$ since $s/k\ge1$, and hence
$q\ge\tfrac12\min\{s/k,m\}$.

The same class on the same support also contains the two constants $\pm kB$: use $k$
zero-weight units with bias $-B$ and a common outer sign. Thus for every sample,
\begin{equation}
\label{eq:bias-lower}
 \Rad_{\Hc}(S)\ge kB\rhoM\ge\frac{kB}{\sqrt{3m}}.
\end{equation}
The two constructions are alternative subclasses; they need not fit simultaneously into one
network. Taking the larger of \eqref{eq:weight-lower} and \eqref{eq:bias-lower} and using
$\max\{a,b\}\ge(a+b)/2$ proves \eqref{eq:lower}. If $A=0$, take
$\mathcal X_0=\{0\}$ and use only the constant subclass; the weight term is zero.
Finally, a supremum over samples is at least their expectation under this marginal.
\end{proof}

\paragraph{Interpretation of the regimes.}
When $m\le s/k$, we may use one cluster per sample-sized support point; random repetitions
only change constants in the lower bound. When $m\ge s/k$, the $\lfloor s/k\rfloor$
clusters are repeatedly observed and their signed sums have square-root fluctuations.
The bias subclass has just one freely chosen sign, giving $m^{-1/2}$ decay independently of
the number of available clusters. This explains the distinct width dependence in the upper
bound. The dimensional condition is used to supply orthogonal cluster directions; it is not
an assumption of Theorem~\ref{thm:upper}, and no claim of sharpness for every smaller fixed
dimension is made.

\section{When does sparsity remove the width dependence?}
\label{sec:geometry}

The lower bound in Section~\ref{sec:lower} allows different neuron clusters to act independently
on a discrete input set. Does its width dependence survive if sparsity is required throughout a
full-dimensional convex domain? The answer depends on whether activation thresholds are available.
Write $\Ball=\{x\in\R^n:\|x\|_2\le R\}$ and suppose $R>0$ throughout this section.

\begin{proposition}[Zero-threshold width collapse]
\label{prop:zero-ball}
Every zero-bias network that is $k$-sparse on $\Ball$ has a reduced representation with at most
$\min\{s,2k\}$ nonzero units. Consequently, for every $S\in(\Ball)^m$,
\[
 \Rad_{\mathcal H_{n,s,k}^{W,0}(\Ball)}(S)\le\frac{4kA}{\sqrt m}.
\]
The supremum over such samples is at least $kA\rhoM$, so its order is $kA/\sqrt m$.
\end{proposition}
\begin{proof}
Use the reduced parameters in \eqref{eq:reduced} and discard $v_j=0$ units.
Choose a unit direction outside the finitely many hyperplanes $\ip{v_j}x=0$.
Each remaining unit is strictly active at exactly one of the two antipodal radius-$R$ points.
Both points lie in the required input domain, so the number of remaining units is at most $2k$.
The neuronwise estimate in Theorem~\ref{thm:upper} then gives the upper bound.
For the lower bound, repeat $Re_1$ and use $k$ parallel units with weight $We_1$ and a common
outer sign. These networks are globally $k$-sparse and take values $\pm kA$ on the sample.
The case $W=0$ is immediate.
\end{proof}

Central symmetry alone is insufficient. On the finite set
$\{\pm Re_b:b\in[q]\}$, the zero-bias clusters from Section~\ref{sec:lower} remain
admissible: most units can vanish on a given antipodal pair. The full-ball argument uses a pair
that avoids every nonzero neuron's zero hyperplane simultaneously.

For a polytope $\mathcal X_0=\operatorname{conv}(V)$ with finitely many vertices, a related
count holds even with biases. Every unit active somewhere is active at a vertex, since an affine
function attains its maximum there. Hence at most $k|V|$ units are nonzero on the domain, giving
the neuronwise bound $2k|V|(A+B)/\sqrt m$. In particular, on an interval ($n=1$), the width
dependence disappears for every bias bound, not only for $B=0$.

Positive thresholds can instead isolate disjoint caps of the ball. Set
\[
 \tau=\frac{\sqrt3}{2},\qquad d_q=\min\{q,\lceil16\log(2q)\rceil\}.
\]
The following elementary packing suffices; we do not need sharp spherical-code bounds.

\begin{lemma}[Separated directions]
\label{lem:directions}
If $n\ge d_q$, there are unit vectors $u_1,\ldots,u_q\in\R^n$ with
$\ip{u_b}{u_c}\le1/2$ for $b\ne c$.
\end{lemma}
\begin{proof}
If $n\ge q$, take orthogonal vectors. Otherwise $n\ge16\log(2q)$.
Take independent vectors uniform on $\{\pm n^{-1/2}\}^n$.
For a fixed pair, their inner product has the law of $n^{-1}\sum_{r=1}^n\xi_r$ for
independent uniform signs. The bound $\E e^{\lambda\sum_r\xi_r}\le e^{n\lambda^2/2}$ gives
$\P[\ip{u_b}{u_c}>1/2]\le e^{-n/8}$.
A union bound over pairs has probability at most $q(q-1)e^{-n/8}/2<1$.
Thus a configuration with the required separation exists.
\end{proof}

\begin{theorem}[A full-ball lower bound in logarithmic dimension]
\label{thm:ball-lower}
Let $q=\min\{\lfloor s/k\rfloor,m\}$ and suppose $n\ge d_q$. Put
$a=\min\{A,B/\tau\}$. There is a fixed marginal $D_x$ on $\Ball$ such that
\begin{equation}
\label{eq:ball-lower}
 \E_{S\sim D_x^m}\Rad_{\Hball}(S)
 \ge c_{\mathrm{cap}}\left(a\min\{k,\sqrt{sk/m}\}+\frac{kB}{\sqrt m}\right),
 \qquad c_{\mathrm{cap}}=\frac{1-\tau}{4\sqrt2}.
\end{equation}
In particular, if $B\ge\tau A$, the rate of Theorem~\ref{thm:lower} holds for the full-ball
class itself, with only $O(\log(2q))$ dimensions sufficient.
\end{theorem}
\begin{proof}
Choose directions from Lemma~\ref{lem:directions} and put $D_x$ uniform on
$\{Ru_1,\ldots,Ru_q\}$. Suppose first that $a>0$.
In cluster $b$, place $k$ copies of the unit
\[
 v_b=(a/R)u_b,\qquad\beta_b=\tau a,
\]
with a common outer sign $\theta_b$. These parameters obey the norm and bias bounds.
If two distinct clusters were active at an input $x\in\Ball$, then
\[
 \ip{u_b+u_c}x>2\tau R,
 \qquad \|u_b+u_c\|_2^2=2+2\ip{u_b}{u_c}\le3=4\tau^2,
\]
contradicting Cauchy--Schwarz. Thus the network is $k$-sparse at \emph{every} point of the ball,
not only at the sampling locations. At $Ru_b$, its value is $ka(1-\tau)\theta_b$; all other
clusters are inactive because $1/2<\tau$. Pad with inactive units to width $s$.

The moment calculation in \eqref{eq:weight-lower}, with amplitude $a(1-\tau)$ replacing $A$,
now gives a lower bound of
$a(1-\tau)\min\{k,\sqrt{sk/m}\}/(2\sqrt2)$.
The same full-ball class contains the constants $\pm kB$, contributing $kB\rhoM$.
Taking the larger of these two subclass bounds proves \eqref{eq:ball-lower}.
If $a=0$, the constant subclass alone proves the assertion.
\end{proof}

\paragraph{What changes between the regimes?}
The zero-bias proposition turns pointwise sparsity into a global bound on the number of useful
units. The cap construction defeats that implication by assigning different clusters disjoint
activation regions. It also shows that a finite-support marginal does not require a function
class whose sparsity promise is restricted to that finite support.
For $m\ge s/k$, the comparison is summarized in Table~\ref{tab:domains}.
The value $\tau A$ is sufficient for the construction, not a proved critical threshold.

The logarithmic dimension order is necessary for constant-amplitude sign encoding.
Indeed, every full-ball network is $kW$-Lipschitz: along any segment, the finitely many
affine pieces have gradients bounded by $kW$. When $A+B>0$, normalization by $k(A+B)$ gives
a function that is at most $1/R$-Lipschitz. If all signs on $q$ inputs can be realized at amplitude
$\alpha>0$, each pair of inputs must be at distance at least $2\alpha R$.
Disjoint open balls of radius $\alpha R$ around them lie inside the ball of radius $(1+\alpha)R$,
so volume comparison gives $q\le(1+1/\alpha)^n$.
This proves an $\Omega(\log q)$ requirement when $\alpha$ is constant; it is a statement about
this encoding, not a necessary dimension condition for every complexity lower bound.

\begin{table}[t]
\centering
\caption{Worst-case empirical complexity; $g=\lfloor s/k\rfloor$ and $m\ge s/k$.
The lower bounds also hold in expectation for an appropriate i.i.d. marginal.
$\widetilde\Theta$ permits logarithms in $m$; the middle row has none.}
\label{tab:domains}
\begin{tabular}{@{}lll@{}}
\toprule
Sparsity domain and bias & Dimension & Complexity \\
\midrule
Finite set, $B=0$ & $n\ge g$ & $\widetilde\Theta(A\sqrt{sk/m})$ \\
Entire ball, $B=0$ & $n\ge1$ & $\Theta(kA/\sqrt m)$ \\
Entire ball, $B\ge\tau A$ & $n\ge d_g$ & $\widetilde\Theta(A\sqrt{sk/m}+kB/\sqrt m)$ \\
\bottomrule
\end{tabular}
\end{table}

\section{Learning consequences}
\label{sec:learning}

Fix the class $\mathcal H=\Hc$ before drawing data. In this section $\mathcal X_0$ is Borel,
$D$ is a probability distribution on $\mathcal X_0\times\R$, and the loss
$\ell:\R\times\R\to[0,b_\ell]$ is jointly Borel measurable and $L$-Lipschitz in its
first argument. Write $\mathcal L_D(h)=\E\ell(h(X),Y)$ and
$\widehat{\mathcal L}_S(h)=m^{-1}\sum_i\ell(h(X_i),Y_i)$. Put
\begin{equation}
\label{eq:learning-U}
 U_m=\min\left\{k(A+B),\frac{2s(A+B)}{\sqrt m},\;
 C A\min\{k,\sqrt{sk/m}\ellm^{3/2}\}+kB\rhoM\right\}.
\end{equation}
This is a deterministic upper bound on the empirical and expected complexity of the fixed
class, for every admissible marginal.

\begin{corollary}[Agnostic and realizable guarantees]
\label{cor:learning}
For any $\eta>0$, there exists a measurable $\eta$-approximate empirical risk minimizer
$\widehat h\in\mathcal H$. For every $\delta\in(0,1)$, with probability at least
$1-\delta$ over $S\sim D^m$ it satisfies
\begin{equation}
\label{eq:agnostic}
 \mathcal L_D(\widehat h)-\inf_{h\in\mathcal H}\mathcal L_D(h)
 \le4L U_m+2b_\ell\sqrt{\frac{\log(2/\delta)}{2m}}+\eta.
\end{equation}
If $Y=h^\star(X)$ almost surely for some $h^\star\in\mathcal H$ and
$\ell(y,y)=0$, then with probability at least $1-\delta$,
\begin{equation}
\label{eq:realizable}
 \mathcal L_D(\widehat h)
 \le2L U_m+b_\ell\sqrt{\frac{\log(1/\delta)}{2m}}+\eta.
\end{equation}
These statements give statistical guarantees, without an efficient optimization claim.
\end{corollary}
\begin{proof}
The parameter space in \eqref{eq:reduced} with the support-wide sparsity constraint is
compact. Indeed, having at most $k$ strictly positive affine values at each fixed $x$ is
a closed condition, and arbitrary intersections of these conditions remain closed in the
bounded parameter box. The finitely many outer sign vectors preserve compactness.
The parameter-to-function map is continuous in the uniform norm on $\mathcal X_0$,
since $\|x\|_2\le R$. Thus $\mathcal H$ is separable in that norm. Choose a fixed
countable dense subclass. Lipschitz continuity of the loss makes its empirical and population
infima agree with those of $\mathcal H$. Selecting the first element of this subclass with
empirical loss within $\eta$ of its countable infimum defines a measurable approximate ERM.
This is an existence argument, not an effective search procedure.

Symmetrization, scalar contraction, and bounded differences give, simultaneously for all
$h\in\mathcal H$, the one-sided bound
\[
 \mathcal L_D(h)-\widehat{\mathcal L}_S(h)
 \le2L\E_{S'}\Rad_{\mathcal H}(S')
       +b_\ell\sqrt{\log(1/\delta)/(2m)}.
\]
Here $S'$ is an independent input sample. This is the expected-complexity form of the
standard Rademacher generalization inequality
\citep{bartlett2002rademacher,mohri2018foundations}. The countable dense subclass justifies
the suprema, and subtracting $\ell(0,Y_i)$ before contraction does not change expected
Rademacher complexity. Apply the same argument in the other direction and use a union
bound to obtain uniform absolute deviation at most
$2L U_m+b_\ell\sqrt{\log(2/\delta)/(2m)}$.
The approximate ERM inequality then proves \eqref{eq:agnostic}.
In the realizable case, $\widehat{\mathcal L}_S(\widehat h)\le\eta$ almost surely;
the single one-sided event proves \eqref{eq:realizable}.
\end{proof}

\subsection{Agnostic minimax risk for a normalized bounded loss}

We now fix one loss and match upper and lower bounds in the same statistical experiment.
Suppose $A>0$, set $g=\lfloor s/k\rfloor$, and normalize the full-ball class as
\[
 \overline{\mathcal H}=\{h/(k(A+B)):h\in\Hball\}\subseteq[-1,1]^{\Ball}.
\]
For $y\in\{-1,1\}$, use the bounded linear loss
\[
 \ell_{\mathrm{lin}}(t,y)=\tfrac12\bigl(1-y\clip(t)\bigr),\qquad
 \clip(t)=\max\{-1,\min\{t,1\}\}.
\]
It is $1/2$-Lipschitz in $t$ and takes values in $[0,1]$; replacing $y$ by $\clip(y)$
extends it to all real labels. Let
\[
 \mathfrak E_m=\inf_{\widehat f}\sup_D
 \left(\E_{S\sim D^m}\mathcal L_D(\widehat f)-
                       \inf_{f\in\overline{\mathcal H}}\mathcal L_D(f)\right),
\]
where $D$ ranges over all Borel distributions on $\Ball\times\{-1,1\}$.
The infimum allows measurable, randomized and improper learning rules; their internal randomness
is included in the expectation. The rule receives the sample, not the unknown distribution.
No computation restriction is imposed.

\begin{theorem}[Agnostic minimax rate]
\label{thm:agnostic-minimax}
Suppose $A>0$, $\tau A\le B\le A$, and $n\ge d_g$, where $\tau,d_g$ are defined in
Section~\ref{sec:geometry}. There are universal constants $c_1,C_1>0$ such that, for all $m\ge1$,
\begin{equation}
\label{eq:agnostic-minimax}
 c_1\min\{1,\sqrt{s/(km)}\}
 \le\mathfrak E_m
 \le C_1\min\{1,\sqrt{s/(km)}\log^{3/2}(2m)\}.
\end{equation}
\end{theorem}

The upper bound follows from approximate empirical risk minimization in the normalized class.
For the lower bound, the spherical-cap construction realizes all independent signs on
$q=\min\{g,m\}$ inputs with a constant normalized amplitude. Give the labels small, unknown
signed means on those inputs. Adjacent sign choices have small joint-sample divergence, so no
learning rule can reliably recover their correlations. Appendix~\ref{app:agnostic} makes this
testing argument explicit, including improper and randomized rules.

For this loss and model, \eqref{eq:agnostic-minimax} yields an expected-excess-risk sample-size
order $\widetilde\Theta(s/(k\varepsilon^2))$ for sufficiently small $\varepsilon$.
The ratio $s/k$ reflects the explicit output normalization by $k(A+B)$; it is not an
unnormalized improvement from $sk$ to $s/k$. The theorem is loss-specific and agnostic.
It does not establish a matching lower bound for the realizable guarantee above or for
every bounded Lipschitz loss.

\section{Scope and remaining questions}
\label{sec:discussion}

Activation sparsity does not by itself identify a globally small subnetwork. The full-ball
comparison makes the distinction precise: zero thresholds force at most $2k$ nonzero units,
whereas positive thresholds can isolate many independently signed caps, even in logarithmic
dimension. Meanwhile, biases larger than the linear pre-activation range add only constants.
These are different roles of thresholds, not conflicting statements about bias complexity.

The logarithmic gap remains. Our affine cover contributes $\sqrt{\log(2m)}$, and summing
the entropy bound over scales contributes another logarithm. The finite union over activation
counts has only a lower-order cost. Removing that union cost alone cannot prove the
logarithm-free conjecture. A direct process estimate, or an analysis retaining the simultaneous
feasibility of all neuron configurations, may avoid losses introduced by the present relaxation.

The domain results leave a more geometric question: how does complexity interpolate between
the zero-bias collapse and the full-ball cap construction? The sufficient value
$B=(\sqrt3/2)WR$ is not shown to be a critical threshold. The constants in the spherical-code
dimension estimate are not sharp. Understanding small positive thresholds and fixed low dimension
would explain more than improving a universal upper bound alone.

Finally, the agnostic minimax characterization concerns one explicitly normalized bounded
loss and comparable bias and weight scales. It does not settle realizable learning or every
loss function. All classes here impose sparsity on a fixed input domain before data are drawn;
observed training-set sparsity is not a substitute for that promise.
Distribution-dependent sparsity and multilayer extensions require additional arguments.

\bibliographystyle{plainnat}
\bibliography{refs}
\newpage
\appendix
\section{The affine covering estimate}
\label{app:affine}

We give the finite-dimensional argument underlying \eqref{eq:affine}. It is the usual
Gaussian packing proof of dual Sudakov
\citep{pajor1986subspaces,ledoux1991probability}, specialized to the sample norm.

\begin{lemma}[Samplewise affine covering]
\label{lem:affine}
For every radius-$R$ sample, every $W,B\ge0$, and every $\delta>0$,
the affine evaluation class satisfies \eqref{eq:affine}.
\end{lemma}
\begin{proof}
We first establish the linear bound. Let $V=\spn\{x_1,\ldots,x_m\}$,
$p(v)=\max_i|\ip v{x_i}|$, and $L_X=\max_i\|x_i\|_2$.
Orthogonal projection onto $V$ preserves the sample evaluations and decreases the Euclidean
norm. If $W=0$ or $L_X=0$, the linear class is a singleton. Otherwise $p$ is a norm on $V$.
Let $G$ be a standard Gaussian in $V$, $a=\E p(G)>0$, and $K=\{v:p(v)\le1\}$.

For any centrally symmetric Borel set $K'$ and $z\in V$, Gaussian density integration gives
\begin{equation}
\label{eq:gaussian-shift}
 \gamma(K'+z)=e^{-\|z\|_2^2/2}
     \int_{K'} e^{-\ip z y}\,d\gamma(y)
 \ge e^{-\|z\|_2^2/2}\gamma(K').
\end{equation}
Indeed, by symmetry the integral equals that of $\cosh(\ip z y)$, which is at least
$\gamma(K')$. This argument applies in the Euclidean space $V$ with its own Gaussian density.

Fix a covering radius $r>0$. Choose a maximal finite collection
$v_1,\ldots,v_N\in W B_2(V)$ with pairwise $p$-distance strictly larger than $r$.
Such a collection exists because the ball is compact. Its closed $r$-balls cover the ball,
and the sets $v_j+(r/2)K$ are disjoint. Put $\lambda=4a/r$.
The dilated sets remain disjoint, and Markov's inequality gives
$\gamma((\lambda r/2)K)=\P[p(G)\le2a]\ge1/2$.
Using \eqref{eq:gaussian-shift} and $\|v_j\|_2\le W$,
\[
 1\ge\sum_{j=1}^N\gamma(\lambda v_j+(\lambda r/2)K)
 \ge\frac N2 e^{-\lambda^2W^2/2}.
\]
It follows that $\log N\le\log2+8W^2a^2/r^2$.
If $r\ge WL_X$, the single center $0$ suffices instead. Otherwise,
$a\ge\sqrt{2/\pi}\,L_X$, by choosing an input attaining $L_X$ and taking the
expected absolute value of its Gaussian projection. Hence
$\log2\le(\pi\log2/2)W^2a^2/r^2$, and in both cases
\begin{equation}
\label{eq:dual-special}
 \log N(WB_2(V),r,p)\le10W^2a^2/r^2.
\end{equation}

Each of the $2m$ Gaussian variables $\pm\ip G{x_i}$ has variance at most $R^2$.
The exponential-moment maximal inequality gives
$a\le R\sqrt{2\log(2m)}$. Setting $r=\delta/2$ in \eqref{eq:dual-special} bounds
the logarithm of the linear covering number by $80A^2\log(2m)/\delta^2$.
Finally, a grid at spacing $\delta/2$, starting at $-B$ and ending at the last such grid
point not exceeding $B$, covers $[-B,B]$ at radius $\delta/2$ using at most
$\lfloor4B/\delta\rfloor+1$ points. Adding this grid to the linear cover makes a
$\delta$-cover of affine evaluations. Cardinalities multiply, proving \eqref{eq:affine}.
\end{proof}

\paragraph{Measurability and zero scales.}
For each fixed sample and $t$, the set of $(v,\beta)$ with at most $t$ positive sample
pre-activations is closed in the compact parameter box: the violating set is the finite
union of conditions under which $t+1$ specified coordinates are strictly positive.
Its continuous image $\mathcal F_t$ is compact and contains zero. The normalized union
$T$ in Section~\ref{sec:upper} is a finite union of compact sets. Its suprema are finite;
the sign probability space is itself finite. The proof handles $M=0$ before invoking a
positive covering radius and chooses at least one chaining level even when $m=1$.

\section{Proof of the agnostic minimax bound}
\label{app:agnostic}

We prove Theorem~\ref{thm:agnostic-minimax} for its normalized class and bounded linear loss.

\paragraph{Upper bound.}
Use a measurable $\eta$-approximate empirical risk minimizer in $\overline{\mathcal H}$,
whose existence follows from the compactness argument for Corollary~\ref{cor:learning}.
The class is symmetric and contains zero. On its prediction range,
$\ell_{\mathrm{lin}}(f(x),y)-1/2=-yf(x)/2$.
Symmetrization and sign symmetry therefore give
\[
 \E\sup_{f\in\overline{\mathcal H}}
       |\mathcal L_D(f)-\widehat{\mathcal L}_S(f)|
 \le \E_{S'}\Rad_{\overline{\mathcal H}}(S')
 \le \frac{U_m}{k(A+B)}.
\]
Multiplying the Rademacher signs by the labels preserves their conditional law; the loss's
factor $1/2$ cancels the symmetrization factor $2$. Approximate ERM thus has expected excess risk
at most $2U_m/(k(A+B))+\eta$, uniformly in $D$.
Using \eqref{eq:learning-U}, $s\ge k$, and $\rhoM\le m^{-1/2}$, and taking the better of
this bound and the trivial bound $1$, proves the upper inequality in
\eqref{eq:agnostic-minimax}. Letting $\eta\downarrow0$ is legitimate in the infimum over
learning rules; no exact measurable minimizer is needed.

\paragraph{Lower bound: one fixed family before sampling.}
Put $q=\min\{g,m\}$. Since $d_q\le d_g$, choose $q$ cap directions as in
Theorem~\ref{thm:ball-lower}. Its construction with $a=A$ yields functions
$f_\theta\in\overline{\mathcal H}$ such that
\[
 f_\theta(Ru_b)=\alpha\theta_b,\qquad
 \theta\in\{-1,1\}^q,\qquad
 \alpha=\frac{(1-\tau)A}{A+B}\ge\frac{1-\tau}{2}.
\]
Set $\delta_0=(\alpha/4)\sqrt{q/m}\le1/4$.
For each fixed $\theta$, define $D_\theta$ by drawing $X$ uniformly from the cap centers and
setting
\[
 \P_\theta[Y=1\mid X=Ru_b]=\frac{1+\delta_0\theta_b}{2}.
\]
This family is fixed before sampling. Let $P_\theta=D_\theta^m$, and let $\theta^{(b)}$
flip coordinate $b$.
The common input marginal gives
\begin{align*}
 \KL(P_\theta\|P_{\theta^{(b)}})
 &=\frac mq\,\delta_0\log\frac{1+\delta_0}{1-\delta_0}\\
 &\le\frac{4m\delta_0^2}{q}.
\end{align*}
Here $\log((1+t)/(1-t))\le4t$ for $0\le t\le1/2$, for example by integrating
$2/(1-t^2)\le4$. Pinsker's inequality implies
\[
 \TV(P_\theta,P_{\theta^{(b)}})
 \le\delta_0\sqrt{2m/q}
 =\frac{\alpha}{2\sqrt2}\le\frac\alpha2.
\]

Fix any learning rule, including an improper or randomized one, and write
$T_b=\clip(\widehat f(Ru_b))\in[-1,1]$.
Averaging over adjacent pairs and using the bounded-observable characterization of total variation,
\[
 2^{-q}\sum_\theta\theta_b\E_\theta T_b
 \le\frac\alpha2.
\]
For deterministic rules, pair each $\theta$ with $\theta_b=1$ with its flipped neighbor and use
$(\E_\theta T_b-\E_{\theta^{(b)}}T_b)/2\le\TV(P_\theta,P_{\theta^{(b)}})$.
Randomized post-processing cannot increase total variation, so the same inequality applies.

The comparator $f_\theta$ has risk $1/2-\delta_0\alpha/2$.
Thus the average, over this fixed family, of the learning rule's expected excess risk is at least
\begin{align*}
 2^{-q}\sum_\theta\left(\E_\theta\mathcal L_{D_\theta}(\widehat f)
                 -\inf_{f\in\overline{\mathcal H}}\mathcal L_{D_\theta}(f)\right)
 &\ge\frac{\delta_0\alpha}{2}
       -\frac{\delta_0}{2q}\sum_{b=1}^q
                  2^{-q}\sum_\theta\theta_b\E_\theta T_b\\*
 &\ge\frac{\delta_0\alpha}{4}
 =\frac{\alpha^2}{16}\sqrt{q/m}.
\end{align*}
A supremum over $D$ is at least this average. Finally,
$q\ge\frac12\min\{s/k,m\}$ and $\alpha\ge(1-\tau)/2$, so one may take
$c_1=(1-\tau)^2/(64\sqrt2)$.
This Assouad-type argument \citep[Chapter~2]{tsybakov2009introduction} directly bounds excess
risk for the full-ball class, rather than converting a complexity lower bound.

\section*{AI Disclosure}
The original manuscript, including its activation-budget reduction and margin-shifted covering
argument, was written by the authors. During subsequent revisions, we used OpenAI Codex,
including GPT-5.6 Solar and GPT-6 Astra, to assist with proof checking, mathematical revisions,
literature lookup, and exposition.

Specifically, the AI-assisted revisions replaced the per-scale concentration and dyadic peeling
steps with a single chaining argument over the union of normalized single-neuron classes
(Section~\ref{sec:upper}); introduced the threshold-clipping decomposition that isolates the
constant contribution and removes the width and logarithmic factors from the bias term
(Lemma~\ref{lem:bias}); extended the clustered lower-bound construction to i.i.d. samples
without divisibility assumptions (Theorem~\ref{thm:lower}); added the zero-bias width-collapse
result and the disjoint-cap lower bound on the entire ball in logarithmic dimension
(Section~\ref{sec:geometry}); and reformulated the learning lower bound as an Assouad-type
argument for the normalized class under a specified bounded linear loss, so that the learning
upper and lower bounds concern the same statistical model
(Theorem~\ref{thm:agnostic-minimax}). AI tools also assisted with edge-case checks, finite-instance
verification scripts, and expansion of the technique overview. We take full responsibility for
all content in the final manuscript, including AI-assisted material, and for its correctness,
originality, attribution, and citations.
\end{document}